\def\TITLE{%
  Parameterized Runtime Analyses of Evolutionary Algorithms for the
  Euclidean Traveling Salesperson Problem 
}

\documentclass[11pt]{article}

\title{\TITLE}
\author{%
  Andrew M. Sutton and Frank Neumann\\
  ~\\
  School of Computer Science,
  University of Adelaide\\ 
  Adelaide, SA 5005, Australia}
\date{}

\usepackage{amsmath, amssymb} 

\usepackage{dsfont}

\def\Na{\mathds{N}}

\def\NP{\mathsf{NP}}
\def\XP{\mathsf{XP}}
\def\FPT{\mathsf{FPT}}
\def\W1{\mathsf{W}[1]}

\newcommand{\hull}[1]{\mathfrak{H}\!\left({#1}\right)}
\newcommand{\inv}[2]{\sigma^{\rm I}_{#1#2}}  
\newcommand{\jmp}[2]{\sigma^{\rm J}_{#1#2}}  

\def\E{\mathbb{E}}

\def\e{e}

\def\opt{\ensuremath{opt}}

\def\fn{\ensuremath{\gamma}}

\usepackage{amsthm}
\newtheorem{theorem}{Theorem}
\newtheorem{lemma}{Lemma}
\newtheorem{proposition}{Proposition}
\newtheorem{definition}{Definition}
\newtheorem*{corollary}{Corollary}

\usepackage{tikz} 

\usepackage{fullpage}
\usepackage[algo2e,ruled,vlined,linesnumbered,procnumbered,onelanguage]{algorithm2e}
\SetKw{Return}{return}

\begin{document}
\maketitle
\begin{abstract}
  Parameterized runtime analysis seeks to understand the influence of
  problem structure on algorithmic runtime. In this paper, we
  contribute to the theoretical understanding of evolutionary
  algorithms and carry out a parameterized analysis of evolutionary
  algorithms for the Euclidean traveling salesperson problem
  (Euclidean TSP).

  We investigate the structural properties in TSP instances that
  influence the optimization process of evolutionary algorithms and
  use this information to bound the runtime of simple evolutionary
  algorithms.  Our analysis studies the runtime in dependence of the
  number of inner points $k$ and shows that $(\mu + \lambda)$
  evolutionary algorithms solve the Euclidean TSP in expected time
  $O((\mu/\lambda) \cdot n^3\fn(\epsilon) + n\fn(\epsilon) +
  (\mu/\lambda) \cdot n^{4k}(2k-1)!)$ where $\fn$ is a function of
  the minimum angle $\epsilon$ between any three points. 

  Finally, our analysis provides insights into designing a mutation
  operator that improves the upper bound on expected runtime. We show
  that a mixed mutation strategy that incorporates both 2-opt moves
  and permutation jumps results in an upper bound of $O((\mu/\lambda)
  \cdot n^3\fn(\epsilon) + n\fn(\epsilon) + (\mu/\lambda) \cdot
  n^{2k}(k-1)!)$ for the $(\mu+\lambda)$~EA.
\end{abstract}

\section{Introduction}\label{sec:introduction}
In many real applications, the inputs of an $\NP$-hard combinatorial
optimization problem may be structured or restricted in such a way
that it becomes tractable to solve in practice despite having a
worst-case exponential time bound. Parameterized analysis seeks to
address this by expressing algorithmic runtime in terms of an
additional hardness parameter that isolates the source of exponential
complexity in the problem structure. In this paper, we study the
application of evolutionary algorithms (EAs) to the Euclidean
Traveling Salesperson Problem (TSP) and consider the runtime of such
algorithms as a function of both problem size and a further parameter
that influences how hard the problem is to solve by an EA.

\subsection{The Euclidean traveling salesperson problem}
Iterative heuristic methods (such as local search and evolutionary
algorithms) that rely on the exchange of a few edges such as the
well-known 2-opt (or 2-change) operator are popular choices for
solving large scale TSP instances in practice. This is partly due to
the fact that they have a simple implementation and typically perform
well empirically.  However, for these algorithms, theoretical
understanding still remains limited. Worst-case analyses demonstrate
the existence of instances on which the procedure can be
inefficient. Chandra, Karloff, and Tovey~\cite{Chandra1999new},
building from unpublished results due to Lueker, have shown that local
search algorithms employing a $k$-change neighborhood operator can
take exponential time to find a locally optimal solution. Even in the
Euclidean case, Englert, R\"oglin, and
V\"ocking~\cite{Englert2007worst} have recently shown that a local
search algorithm employing inversions can take worst-case exponential
time to find tours which are locally optimum.

If the search operator is restricted to specialized 2-opt moves that
remove only edges that intersect in the plane, van Leeuwen and
Schoone~\cite{vanLeeuwen1981untangling} proved that a tour that has no
such planar intersections can be reached in $O(n^3)$ moves, even if a
move introduces further intersecting edges. Since determining which
edges are intersecting can take quadratic time, a locally optimal tour
can be found in time $O(n^5)$.  We point out that a local optimum in
this restricted neighborhood does not necessarily correspond to a
local optimum in the general 2-opt neighborhood.

If the vertices are distributed uniformly at random in the unit
square, Chandra, Karloff, and Tovey~\cite{Chandra1999new} showed that
the expected time to find a locally optimal solution is bounded by
$O(n^{10} \log n)$. More generally, for so-called $\phi$-perturbed
Euclidean instances, Englert, R\"oglin, and
V\"ocking~\cite{Englert2007worst} proved that the expected time to
find a locally optimum solution is bounded by $O(n^{4 + 1/3}
\log(n\phi) \phi^{8/3})$. These results also imply similar bounds for
simple ant colony optimization algorithms as shown in
\cite{AntsTsp12}.

To allow for a deeper insight into the relationship between problem
instance structure and algorithmic runtime, we appeal in this paper to
the theory of parameterized complexity~\cite{RodFell1999}.  Rather
than expressing the runtime solely as a function of problem size,
parameterized analysis decomposes the runtime into further parameters
that are related to the structure of instances. The idea is to find
parameters that partition off the combinatorial explosion that leads
to exponential runtimes~\cite{DowneyFellows1999after}.

In the context of TSP, a number of parameterized results currently
exist.  De{\u{\i}}neko et al.~\cite{Deineko2006inner} showed that, if
a Euclidean TSP instance with $n$ vertices has $k$ vertices interior
to the convex hull, there is a dynamic programming algorithm that can
solve the instance in time bounded by $g(k)\cdot n^{O(1)}$ where $g$
is a function that depends only on $k$. This means that this
parameterization belongs to the complexity class $\FPT$, the class of
parameterized problems that are considered fixed-parameter tractable.
Of course, membership in $\FPT$ depends strongly on the
parameterization itself. For example, the problem of searching the
$k$-change neighborhood for metric TSP is hard for $\W1$ due to
Marx~\cite{Marx2008tsp}. Therefore, the latter parameterization is not
likely to belong to $\FPT$.

\subsection{Computational complexity of evolutionary algorithms}
Initial studies on the computational complexity of evolutionary
algorithms consider their runtime on classes of artificial
pseudo-Boolean functions
\cite{DJWoneone,DBLP:journals/ai/HeY01,DBLP:journals/ai/HeY03,DBLP:conf/aaai/YuZ06}. The
goal of these studies is to consider the impact of the different
modules of an evolutionary algorithm and to develop new methods for
their analysis.  This early work was instrumental in establishing a
rigorous understanding of the behavior of evolutionary algorithms on
simple functions, for identifying some classes of problems that simple
EAs can provably solve in expected polynomial time~\cite{DJWoneone},
and for disproving widely accepted conjectures (e.g., that
evolutionary algorithms are always efficient on unimodal
functions~\cite{DBLP:conf/ppsn/DrosteJW98}).

More recently, classical polynomial-time problems from combinatorial
optimization such as minimum spanning
trees~\cite{NeumannWegener2005,NeumannWegenerTCS07} and shortest
paths~\cite{STWsorting,DBLP:conf/cec/DoerrHK07,BBDFKN09} have been
considered.  In this case, one does not hope to beat the best
problem-specific algorithms for classical polynomial solvable
problems. Instead, these studies provide interesting insights into the
search behavior of these algorithms and show that many classical
problems are solved by general-purpose algorithms such as evolutionary
algorithms in expected polynomial time.

Research on $\NP$-hard combinatorial optimization problems such as
makespan scheduling, covering problems, and multi-objective minimum
spanning trees~\cite{Neu07EJOR,WittWorstCaseAverageCase} show that
evolutionary algorithms can achieve good approximations for these
problems in expected polynomial time.  In the case of the TSP,
Theile~\cite{Theile2009exact} has proved that a $(\mu+1)$~EA based on
dynamic programming can exactly solve the TSP in at most $O(n^3 2^n)$
steps when $\mu$ is allowed to be exponential in $n$.  For a
comprehensive presentation of the different results that have been
achieved see, e.g., the recent text of Neumann and Witt
\cite{BookNeuWit}.

Algorithmic runtime on $\NP$-hard problems can be studied in much
sharper detail from the perspective of parameterized analysis, and
this has only recently been started in theoretical work on
evolutionary algorithms. Parameterized results have been obtained for
the vertex cover problem~\cite{KratschNeumannGECCO09}, the problem of
computing a spanning tree with a maximal number of
leaves~\cite{DBLP:conf/ppsn/KratschLNO10}, variants of maximum
2-satisfiability~\cite{Sutton2012max2sat}, and makespan
scheduling~\cite{Sutton2012makespan}.

\subsection{Our results}
In this paper, we carry out a parameterized complexity analysis for
evolutionary algorithms for the Euclidean TSP\@. We prove upper bounds
on the expected runtime of two classical EAs based on 2-opt mutation
in the context of the TSP parameterization of De{\u{\i}}neko et
al.~\cite{Deineko2006inner}, that is, as a function of the number of
points that lie on the interior of the convex hull. Our results are
for the $(\mu+\lambda)$~EA which operates on a population of $\mu$
permutations (candidate Hamiltonian cycles) and produces $\lambda$
offspring in each generation using a mutation operator based on 2-opt.
This analysis provides further insights into the optimization process
that allows us to design a mixed mutation operator that uses both
2-opt moves and permutation jumps and improves the upper bound on
expected runtime of the $(\mu+\lambda)$~EA.

By setting $\mu=\lambda=1$ and changing the mutation operator to
single random 2-opt moves, we also prove parameterized runtime bounds
for randomized local search (RLS): a randomized hill-climber on the
space of permutations. In this case, we present results for the
expected time for RLS to converge to a locally optimal tour in terms
of 2-opt moves. Specialized results for RLS and the $(1+1)$~EA using
2-opt mutation appear in a conference version of this paper presented
at AAAI 2012~\cite{Sutton2012tsp}.

The paper is organized as follows. In Section~\ref{sec:preliminaries}
we introduce the problem, algorithm and analysis. In
Section~\ref{sec:struct-prop} we study structural properties of the
Euclidean TSP\@. In Section~\ref{sec:inst-conv-posit} we study the
runtime of the $(\mu+\lambda)$~EA and RLS on Euclidean TSP instances
whose points lie in convex position, i.e., have no inner points. In
Section~\ref{sec:param-runt-analys} we then prove rigorous runtime
bounds for the algorithms as a function of the number of inner points
in an instance. We conclude the paper in Section~\ref{sec:conclusion}.

\section{Preliminaries}\label{sec:preliminaries}
Let $V$ be a set of $n$ points in the plane labeled as $[n] =
\{1,\ldots,n\}$ such that no three points are collinear. We consider
the complete, weighted Euclidean graph $G = (V,E)$ where $E$ is the
set of all $2$-sets from $V$. The weight of an edge $\{u,v\} \in E$ is
equal to $d(u,v)$: the Euclidean distance separating the points. The
goal is to find a set of $n$ edges of minimum weight that form a
Hamiltonian cycle in $G$. A candidate solution of the TSP is a
permutation $x$ of $V$ which we consider as a sequence of distinct
elements $x = (x_1, x_2, \ldots, x_n), \textrm{ such that } x_i \in
[n]$.  The Hamiltonian cycle in $G$ induced by such a permutation is
the set of $n$ edges
\[
C(x) = \left\{ \{x_1, x_2\}, \{x_2, x_3\}, \ldots, \{x_{n-1}, x_n\},
  \{x_n, x_1\} \right\}.
\]
The optimization problem is to find a permutation $x$ which minimizes
the fitness function
\begin{equation}
  \label{eq:fitness}
  f(x) = \sum_{\{u,v\} \in C(x)} d(u,v).
\end{equation}

Geometrically, it will often be convenient to consider an edge
$\{u,v\}$ as the unique planar line segment with end points $u$ and
$v$.  We say a pair of edges $\{u,v\}$ and $\{s,t\}$ \emph{intersect}
if they cross at a point in the Euclidean plane. An important
observation, which we state here without proof, is that any pair of
intersecting edges form the diagonals of a convex quadrilateral in the
plane (see Figure~\ref{fig:intersection}).

\begin{proposition}
\label{prp:intersection}
If $\{u,v\}$ and $\{s,t\}$ intersect at a point $p$, they form the
diagonals of a convex quadrilateral described by points $u$, $s$, $v$,
and $t$. Hence edges $\{s,u\}$, $\{s,v\}$, $\{t,v\}$ and $\{t,u\}$
form a set of edges that mutually do not intersect.
\end{proposition}

\begin{figure}
  \centering  
  \begin{tikzpicture}[scale=0.6]
    \draw[very thick] (0,8) edge (11,4);
    \draw[very thick] (4,9) edge (8,0);
    \draw[dashed] (0,8) -- (4,9) -- (11,4) -- (8,0) -- cycle;
    \draw[fill=white] (0,8) circle (0.3cm);
    \draw[fill=white] (4,9) circle (0.3cm);
    \draw[fill=white] (11,4) circle (0.3cm);
    \draw[fill=white] (8,0) circle (0.3cm);
    \node[draw=none,above=0.25cm] at (0,8) {$u$};
    \node[draw=none,above=0.25cm] at (4,9) {$s$};
    \node[draw=none,above=0.25cm] at (11,4) {$v$};
    \node[draw=none,below=0.25cm] at (8,0) {$t$};
    \node[draw=none] at (5.5,6.5) {$p$};
  \end{tikzpicture}
  \caption{\label{fig:intersection} Edges that intersect at a point
    $p$ form the diagonals of a convex quadrilateral in the plane.}
\end{figure}
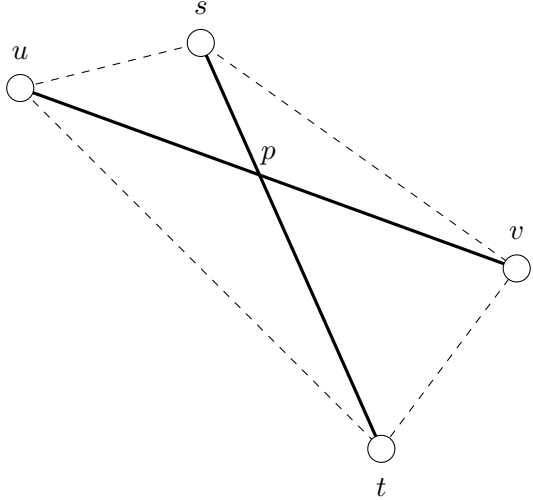

\begin{definition}
  A tour $C(x)$ is called \emph{intersection-free} if it contains no
  pairs of edges that intersect.
\end{definition}

\subsection{Parameterized analysis}\label{sec:param-analys}
Parameterized complexity theory is an extension to traditional
computational complexity theory in which the analysis of hard
algorithmic problems is decomposed into parameters of the problem
input. This approach illuminates the relationship between hardness and
different aspects of problem structure because it often isolates the
source of exponential complexity in $\NP$-hard problems.

A parameterization of a problem is a mapping of problem instances into
the set of natural numbers. We are interested in expressing
algorithmic complexity in terms of both problem size and the extra
parameter. Formally, let $L$ be a language over a finite alphabet
$\Sigma$. A \emph{parameterization} of $L$ is a mapping $\kappa :
\Sigma^* \to \Na$. The corresponding \emph{parameterized problem} is
the pair $(L,\kappa)$.

For a string $x \in \Sigma^*$, let $k = \kappa(x)$ and $n = |x|$. An
algorithm deciding $x \in L$ in time bounded by $g(k) \cdot n^{O(1)}$
is called a \emph{fixed-parameter tractable} (or fpt-) algorithm for
the parameterization $\kappa$. Here $g : \Na \to \Na$ is an arbitrary
but computable function. Similarly, an algorithm that decides a
parameterized problem $(L,\kappa)$ in time bounded by $n^{g(k)}$ is
called an $\XP$-algorithm. 

When working with the runtime of randomized algorithms such as
evolutionary algorithms, one is often interested in the random
variable $T$ which somehow measures the number of steps the algorithm
must take to decide a parameterized problem.  A randomized algorithm
with \emph{expected} optimization time $E(T) \leq g(k) \cdot n^{O(1)}$
(respectively, $E(T) \leq n^{g(k)}$) is a randomized fpt-algorithm
(respectively, $\XP$-algorithm) for the corresponding parameterization
$\kappa$.

In the case of the Euclidean TSP on a set of points $V$, we want to
express the runtime complexity as a function of $n$ and $k$ where $n =
|V|$ and $k$ is the number of vertices that lie in the interior of the
convex hull of $V$. We will hereafter refer to these points as the
\emph{inner points} of $V$. For the corresponding optimization
problem, we are interested in the runtime until the optimal solution
is located.

\subsection{Simple evolutionary algorithms}\label{sec:runt-analys-simple}
Randomized search heuristics such as randomized local search and
evolutionary algorithms that are tasked to solve TSP instances are
usually designed to iteratively search the space of permutations in
order to minimize the fitness function defined in
Equation~(\ref{eq:fitness}).  Each permutation corresponds to a
particular Hamiltonian cycle in the graph.  To move through the space
of candidate solutions, move or mutation operators are often
constructed based on some kind of elementary operations on the set of
permutations on $[n]$.  In this paper, we will consider two such
operations: \emph{inversions} and \emph{jumps} which we define as
follows.

\begin{definition}
  \label{def:inversion}
  The \emph{inversion} operation $\inv{i}{j}$ transforms permutations
  into one another by segment reversal. A permutation $x$ is
  transformed into a permutation $\inv{i}{j}[x]$ by inverting the
  subsequence in $x$ from position $i$ to position $j$ where $1 \leq i
  < j \leq n$.
  \begin{align*}  
    x &= (x_1, \ldots, x_{i-1}, x_i, x_{i+1},\ldots,x_{j-1},x_j,x_{j+1}, \ldots, x_n)\\
    \inv{i}{j}[x] &= (x_1, \ldots, x_{i-1}, x_j,
    x_{j-1},\ldots,x_{i+1},x_i,x_{j+1}, \ldots, x_n)
  \end{align*}
\end{definition}
In the space of Hamiltonian cycles, the permutation inversion
operation is essentially identical to the well-known 2-opt (or
2-change) operation for TSP\@.  The usual effect of the inversion
operation is to delete the two edges $\{x_{i-1},x_{i}\}$ and
$\{x_{j},x_{j+1}\}$ from $C(x)$ and reconnect the tour
$C(\inv{i}{j}[x])$ using edges $\{x_{i-1},x_j\}$ and
$\{x_{i},x_{j+1}\}$ (see Figure~\ref{fig:inversion}).  Here and
subsequently, we consider arithmetic on the indices to be modulo $n$,
i.e., $1-1 = n$ and $n+1 = 1$.  Since the underlying graph $G$ is
undirected, when $(i,j) = (1,n)$, the operation has no effect because
the current tour is only reversed. There is also no effect when $(i,j)
\in \{(2,n),(1,n-1)\}$. In this case, it is straightforward to check
that the edges removed from $C(x)$ are equal to the edges replaced to
create $C(\inv{i}{j}[x])$.

\begin{figure}
  \centering

  \begin{tikzpicture}[scale=0.6]
    \draw[very thick] (1,7) edge (4,8);
    \draw[very thick] (1,1) edge (4,0);
    \draw[very thick] (4,8) edge (7,7);
        \draw[very thick] (4,0) edge (7,1);
        
        \draw[very thick,dashed] (1,1) .. 
        controls (-0.2,1.5) 
        .. (1,2);        
        \draw[very thick,dashed] (0,4) .. 
        controls (-2.5,3) and (2,3) 
        .. (1,2);
        \draw[very thick,dashed] (0,4) .. 
        controls (3,5) and  (-2,6)
        .. (1,7);
        
        \draw[very thick,dashed] (7,1) .. 
        controls (8.2,1.5) 
        .. (7,2);
        
        \draw[very thick,dashed] (8,4) .. 
        controls (10.5,3) and (6,3)
        .. (7,2);
        \draw[very thick,dashed] (8,4) .. 
        controls (5,5) and (10,6)
        .. (7,7);        
        
        \draw[fill=white] (4,0) circle (0.3cm);
        \draw[fill=white] (4,8) circle (0.3cm);
        \draw[fill=white] (0,4) circle (0.3cm);
        \draw[fill=white] (8,4) circle (0.3cm);        
        \draw[fill=white] (1,1) circle (0.3cm);
        \draw[fill=white] (1,7) circle (0.3cm);
        \draw[fill=white] (7,7) circle (0.3cm);
        \draw[fill=white] (7,1) circle (0.3cm);
        
        \node[draw=none,above=0.25cm] at (1,7) {$x_{i-1}$};
        \node[draw=none,above=0.25cm] at (4,8) {$x_{i}$};
        \node[draw=none,above=0.25cm] at (7,7) {$x_{i+1}$};
        
        \node[draw=none,below=0.25cm] at (1,1) {$x_{j+1}$};
        \node[draw=none,below=0.25cm] at (4,0) {$x_{j}$};
        \node[draw=none,below=0.25cm] at (7,1) {$x_{j-1}$};
        
        \node[draw=none,left=0.25cm] at (0,4) {$x_{1}$};
        \node[draw=none,right=0.25cm] at (8,4) {$x_{n}$};
        
      \end{tikzpicture}
      \hfill
       \begin{tikzpicture}[scale=0.6]
         \draw[very thick] (1,7) edge (4,0);
         \draw[very thick] (1,1) edge (4,8);
         \draw[very thick] (4,8) edge (7,7);
         \draw[very thick] (4,0) edge (7,1);
         
         \draw[very thick,dashed] (1,1) .. 
         controls (-0.2,1.5) 
         .. (1,2);        
         \draw[very thick,dashed] (0,4) .. 
         controls (-2.5,3) and (2,3) 
         .. (1,2);
         \draw[very thick,dashed] (0,4) .. 
         controls (3,5) and  (-2,6)
         .. (1,7);
        
         \draw[very thick,dashed] (7,1) .. 
         controls (8.2,1.5) 
         .. (7,2);
         
         \draw[very thick,dashed] (8,4) .. 
         controls (10.5,3) and (6,3)
         .. (7,2);
         \draw[very thick,dashed] (8,4) .. 
         controls (5,5) and (10,6)
         .. (7,7);        
         
         \draw[fill=white] (4,0) circle (0.3cm);
         \draw[fill=white] (4,8) circle (0.3cm);
         \draw[fill=white] (0,4) circle (0.3cm);
         \draw[fill=white] (8,4) circle (0.3cm);        
         \draw[fill=white] (1,1) circle (0.3cm);
         \draw[fill=white] (1,7) circle (0.3cm);
         \draw[fill=white] (7,7) circle (0.3cm);
         \draw[fill=white] (7,1) circle (0.3cm);
         
         \node[draw=none,above=0.25cm] at (1,7) {$x_{i-1}$};
         \node[draw=none,above=0.25cm] at (4,8) {$x_{i}$};
         \node[draw=none,above=0.25cm] at (7,7) {$x_{i+1}$};
         
         \node[draw=none,below=0.25cm] at (1,1) {$x_{j+1}$};
         \node[draw=none,below=0.25cm] at (4,0) {$x_{j}$};
         \node[draw=none,below=0.25cm] at (7,1) {$x_{j-1}$};
         
         \node[draw=none,left=0.25cm] at (0,4) {$x_{1}$};
         \node[draw=none,right=0.25cm] at (8,4) {$x_{n}$};
      
       \end{tikzpicture}
       
  \caption{\label{fig:inversion} The effect of the inversion
    operation on a Hamiltonian cycle.}
\end{figure}
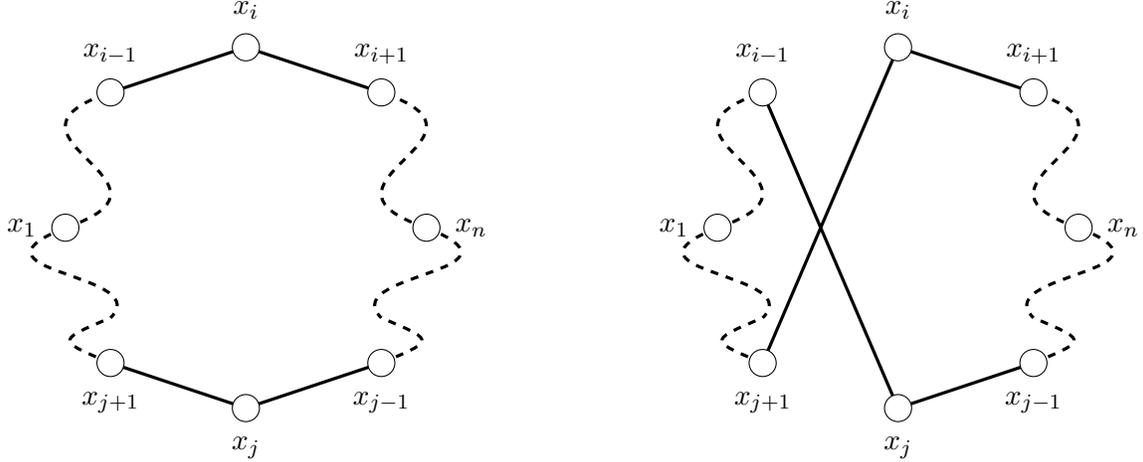

\begin{definition}
  \label{def:jump}
  The \emph{jump} operation $\jmp{i}{j}$ transforms permutations into
  one another by position shifts. A permutation $x$ is transformed
  into a permutation $\jmp{i}{j}[x]$ by moving the element in position
  $i$ into position $j$ while the other elements between position $i$
  and position $j$ are shifted in the appropriate direction.  Without
  loss of generality, suppose $i < j$. Then,
  \begin{align*}  
    x &= (x_1, \ldots, x_{i-1}, x_i, x_{i+1},\ldots,x_{j-1},x_j,x_{j+1}, \ldots, x_n)\\
    \jmp{i}{j}[x] &= 
    (x_1, \ldots, x_{i-1}, x_{i+1},\ldots,x_{j-1}, x_i, x_j,x_{j+1}, \ldots, x_n)\\
    \jmp{j}{i}[x] &= 
    (x_1, \ldots, x_{i-1}, x_j, x_i, x_{i+1},\ldots,x_{j-1},x_{j+1}, \ldots, x_n)
  \end{align*}
\end{definition}
Since $\jmp{i}{(i+1)}$ and $\jmp{(i+1)}{i}$ have the same effect,
there are $n(n-1) - (n-1) = (n-1)^2$ unique jump operations.  The
\emph{jump} operator $\jmp{i}{j}$ was used by Scharnow, Tinnefeld and
Wegener~\cite{STWsorting} in the context of runtime analysis
of evolutionary algorithms on permutation sorting problems.

In this paper, we consider simple mutation-only evolutionary
algorithms that operate on permutations as follows.
\begin{description}
\item[1. Initialization:] Generate a set of $\mu$ permutations on
  $[n]$ uniformly at random.
\item[2. Mutation:] Generate a set of $\lambda$ ``offspring''
  permutations by applying some type of randomized move or mutation
  operator based on the above operations.
\item[3. Selection:] Select the fittest $\mu$ permutations out of the
  both parent and offspring population and continue at line 2.
\end{description}

The general form of such a mutation-only evolutionary algorithm is
typically called a $(\mu+\lambda)$~EA\@.  Algorithm~\ref{alg:ea}
illustrates the $(\mu+\lambda)$~EA\@.  Note that here, the
\texttt{select} function selects the $\mu$ best (with respect to
fitness) permutations from the parent population $P$ and the offspring
population $P'$. This selection mechanism ensures that best-so-far
individual found by generation $t$ remains in the population at
generation $t$, i.e., the $(\mu+\lambda)$~EA exhibits \emph{elitism}.

\begin{algorithm2e}
  \SetKwFunction{mutate}{mutate}
  \SetKwFunction{select}{select}
  \SetKwFor{For}{repeat}{}{}
  Choose a multiset $P$ of $\mu$ random permutations on $[n]$\;
  \For{forever}{%
    $P' \gets \{\}$\;
    \For{$\lambda$ times}{%
      choose $x$ uniformly at random from $P$\;
      $y \gets \mutate(x)$\;
      $P' \gets P' \uplus \{y\}$\;
    }
    $P \gets \select(P \uplus P')$ \;
  }
  \caption{The $(\mu+\lambda)$~EA.}
  \label{alg:ea}
\end{algorithm2e}

To generate offspring using the previously introduced inversion
operation, the $(\mu+\lambda)$~EA employs a mutation operator that
applies a number of random 2-opt moves that is drawn from a Poisson
distribution.  This mutation operator, called
\FuncSty{2-opt-mutation}, is outlined in
Function~\ref{fun:2-opt-mutation}.

\begin{function}
\SetKwInOut{Input}{input}
\SetKwInOut{Output}{output}
\SetProcNameSty{texttt}
\Input{A permutation $x$}
\Output{A permutation $y$}
\BlankLine
$y \gets x$\;
draw $s$ from a Poisson distribution with parameter $1$\; \label{li:Poisson}
perform $s+1$ random inversion operations on $y$\;
\Return $y$;
\caption{2-opt-mutation($x$)} \label{fun:2-opt-mutation}
\end{function}

If we restrict the mutation operation to a single inversion move and
set $\mu=\lambda=1$, the resulting algorithm is \emph{randomized local
  search} (RLS), which is simply a randomized hill-climber in the
space of permutations using 2-opt moves. RLS, illustrated in
Algorithm~\ref{alg:rls}, operates by iteratively applying random
inversion operations to a permutation in order to try and improve the
fitness of the corresponding tours.  Unlike the $(\mu+\lambda)$~EA,
RLS can only generate immediate inversion neighbors so it can become
trapped in local optima.

\begin{algorithm2e}
  \SetKwFor{For}{repeat}{}{}
  Choose a random permutation $x$ on $[n]$\;
  \For{forever}{%
    choose a random distinct pair of elements $(i,j)$ from $[n]$\;
    $y \gets \inv{i}{j}[x]$\;
    \lIf{$f(y) \leq f(x)$}{$x \gets y$}
  }
  \caption{Randomized Local Search (RLS).}
  \label{alg:rls}
\end{algorithm2e}

\subsection{Runtime analysis}
Evolutionary algorithms are simply computational methods that rely on
random decisions so we consider them here as special cases of
\emph{randomized algorithms}.  To analyze the running time of such an
algorithm, we examine the sequence of best-so-far solutions it
discovers during execution
\[
(x^{(1)},x^{(2)},\ldots,x^{(t)},\ldots)
\]
as an infinite stochastic process where $x^{(t)}$ denotes the best
permutation (in terms of fitness) in the population at iteration $t$.
The goal of runtime analysis is to study the random variable that
equals the first time $t$ when $x^{(t)}$ is a candidate solution of
interest (for example, an optimal solution).

The \emph{optimization time} of a randomized algorithm is a random
variable
\begin{equation}
  \label{eq:T}
  T = \inf \{t \in \Na : f(x^{(t)}) \text{~is optimal}\}.
\end{equation}
In the case of the $(\mu+\lambda)$~EA, this corresponds to the number
of generations (iterations of the mutation, evaluation, selection
process) that occur before an optimal solution has been introduced to
the population. This is somewhat distinct from the traditional measure
of the number of explicit calls to the fitness
function~\cite{AugerDoerr2011theory, BookNeuWit}. However, in the case
of the $(\mu+\lambda)$~EA, this metric can be obtained from $T$ by
$T_{f} = \mu + \lambda T$, since we need $\mu$ fitness function calls
to evaluate the initial population and each generation requires
evaluating an additional $\lambda$ individuals.  We discuss this
further in section~\ref{sec:param-runt-analys}.

In this paper, we will estimate the \emph{expected optimization time}
of the $(\mu+\lambda)$~EA\@. This quantity is calculated as $E(T)$,
the expectation of $T$. Since RLS can become trapped in local optima,
its expected optimization time is not necessarily finite. In this
case, we introduce the random variable
\begin{equation}
  \label{eq:Tloc}
  T_{loc} = \inf \{t \in \Na : x^{(t)} \text{~has no improving 2-opt neighbors}\}.
\end{equation}
and estimate $E(T_{loc})$, i.e., the expected time RLS takes to reach
a locally optimal solution.

\begin{definition}
  Let $\alpha$ be an indicator function defined on permutations of
  $[n]$ as
  \[
  \alpha(x) = 
  \begin{cases}
    1 & \text{if $C(x)$ contains intersections;}\\
    0 & \text{otherwise.}
  \end{cases}
  \]
\end{definition}

\begin{definition}
  Let $\beta$ be an indicator function defined on permutations of
  $[n]$ as
  \[
  \beta(x) =
  \begin{cases}
    1 - \alpha(x) & \text{if $f(x)$ non-optimal}\\
    0 & \text{otherwise.}
  \end{cases}
  \]
\end{definition}

The random variable corresponding to optimization time can be
expressed as the infinite series
\begin{equation}
  \label{eq:alphabeta}
  T = \sum_{t = 1}^\infty \bigl(\alpha(x^{(t)}) + \beta(x^{(t)})\bigr).
\end{equation}

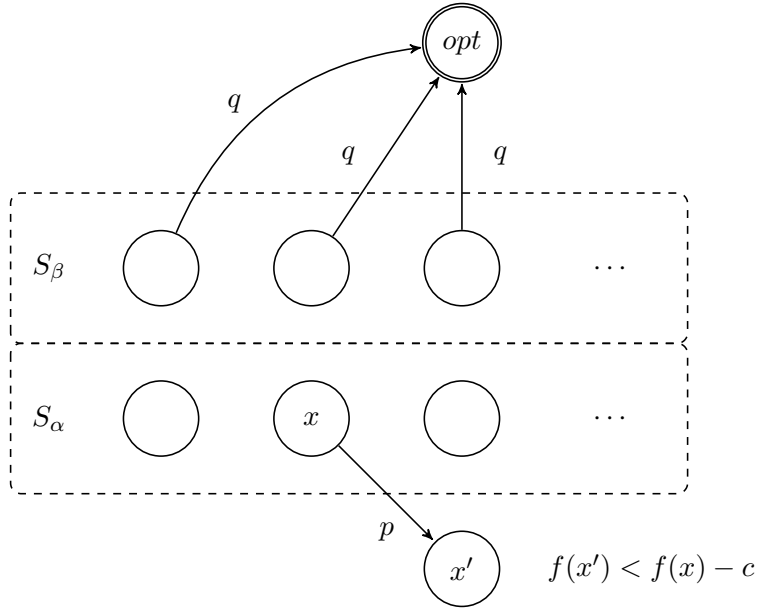
\begin{figure}
  \begin{center}
    \usetikzlibrary{arrows,automata}

    \begin{tikzpicture}[->,>=stealth',shorten >=1pt,auto,semithick,minimum size=10mm]

      \draw[dashed, rounded corners] (0,0) rectangle (9,2);
      \node[right] (beta) at (0,1) {$S_{\beta}$};
      \node[draw,circle] (b1) at (2,1) {};
      \node[draw,circle] (b2) at (4,1) {};
      \node[draw,circle] (b3) at (6,1) {};
      \node           (bdots) at (8,1) {$\cdots$};
      \node[draw,double,circle] (opt) at (6,4) {$\opt$};
      \draw[->] (b1) edge [bend left] node  [left] {$q$} (opt);
      \draw[->] (b2) edge node [left] {$q$} (opt);
      \draw[->] (b3) edge node [right] {$q$} (opt);

      \draw[dashed, rounded corners] (0,0) rectangle (9,-2);
      \node[right] (alpha) at (0,-1) {$S_{\alpha}$};
      \node[draw,circle] (a1) at (2,-1) {};
      \node[draw,circle] (a2) at (4,-1) {$x$};
      \node[draw,circle] (a3) at (6,-1) {};
      \node           (adots) at (8,-1) {$\cdots$};
      \node[draw,circle] (a4) at (6,-3) {$x'$};
      \node[right] (note) at (7,-3) {$f(x') < f(x) - c$};
      \draw[->] (a2) edge node [below] {$p$} (a4);

\end{tikzpicture}
  \end{center}
  \caption{\label{fig:markov} Partitioned Markov chain with transition
    probability bounds. Note that $x'$ can be in either $S_{\alpha}$
    or $S_{\beta}$.}
\end{figure}

In order to characterize the behavior of evolutionary algorithms and
express their expected runtime in terms of the number of points $n$
and the number of inner points $k$, we analyze the Markov chain
generated by the algorithm. We construct the Markov chain as
follows. Given a point set $V$, Each permutation $x$ on $[n]$ that
corresponds to a tour $C(x)$ that is non-optimal is a unique state in
the Markov chain. Finally, every permutation that corresponds to an
optimal tour in $V$ is associated with a single absorbing state
$\opt$. We then bipartition the state space (minus $\opt$) into two
sets $S_{\alpha}$ and $S_{\beta}$ where
\[
S_{\alpha} = \{ x : C(x) \text{~contains intersecting edges}\},
\]
and
\[
S_{\beta} = \{ x : C(x) \text{~is intersection-free~}\} \setminus \{\opt\},
\]
In terms of the Markov chain, the optimization time $T$ is the first
hitting time of the state $\opt$. We will need the following two
preparatory lemmas, the first of which is analogous to the additive
drift result due to He and Yao~\cite{DBLP:journals/ai/HeY01}.
\begin{lemma}
  \label{lem:alpha}
  If there are constants $0 < p < 1$ and $c > 0$ such that for any $x
  \in S_{\alpha}$, the transition probability from $x$ to some $x'$
  with $f(x') < f(x) - c$ is bounded below by $p$, then
  \[
  \E \left( \sum_{t=1}^\infty \alpha(x^{(t)}) \right) \leq p^{-1} n(d_{max} - d_{min})/c,
  \]
  where $d_{max}$ and $d_{min}$ are the maximum and minimum distances
  between any two points in $V$, respectively.
\end{lemma}
\begin{proof}
  By hypothesis, since $p > 0$ for any $x \in S_{\alpha}$ there is a
  nonzero probability the algorithm exits the state $x$ and never
  returns since the sequence of best-so-far solutions increases
  monotonically in fitness. Therefore, the expectation in the claim
  exists and is finite.

  Consider the finite stochastic process $(y^{(1)}, y^{(2)}, \ldots,
  y^{(m)})$, which is defined as the restriction of $(x^{(1)},x^{(2)},\ldots)$
  constructed by taking only permutations $y^{(t)} \in S_{\alpha}$ in
  the same order. It follows that 
  \[
  \E\left(\sum_{t=1}^{\infty} \alpha(x^{(t)}) \right) = \E(m).
  \]
  Let $y^{(1)}$ be the first state in the restricted stochastic
  process. Since every state has a transition to a state that improves
  the fitness by at least $c$ with transition probability at least
  $p$, the expected number of states $y^{(t)}$ with $f(y^{(1)}) \geq
  f(y^{(t)}) \geq f(y^{(1)}) - c$ is bounded by $p^{-1}$. Continuing
  this argument, the expected waiting time until the fitness of any
  state $y^{(t)}$ is improved by at least $c$ is at most $p^{-1}$.

  For any arbitrary permutation $x$, $n d_{max} \geq f(x) \geq n
  d_{min}$ so there can be at most $n(d_{max} - d_{min})/c$ such
  improvements possible.
\end{proof}

\begin{lemma}
  \label{lem:beta}
  If there is a constant $0 < q < 1$ such that for any $x \in
  S_{\beta}$, the transition probability from $x$ to $\opt$ is bounded
  below by $q$, then
  \[
  \E \left( \sum_{t=1}^\infty \beta(x^{(t)}) \right) \leq q^{-1}
  \]
\end{lemma}
\begin{proof}
  Again, since $q > 0$ for any $x \in S_{\beta}$ there is a nonzero
  probability the algorithm exits the state $x$ and transits to the
  absorbing state $\opt$. This event is characterized as an
  independent Bernoulli trial with success probability at least $q$.
  It follows that the series in the claim can be estimated by a
  geometrically distributed random variable with parameter $q$, and
  therefore the expected time spent in states contained in the
  $S_{\beta}$ partition is bounded above by $q^{-1}$.
\end{proof}
In the next section, we will carefully analyze properties of the
Euclidean TSP to find suitable values for $p$, $q$, and $c$. This will
allow us to bound the expected runtime in terms of $n$ and $k$.

\section{Structural properties}\label{sec:struct-prop}
We now examine some useful structural properties of Euclidean TSP
instances related to the inversion operator. We also introduce some
structural constraints that will later facilitate the parameterized
analysis.  We begin by pointing out that if a tour is not
intersection-free, an intersection can always be removed by an
inversion. This notion is captured by the following lemma.
\begin{lemma}
  \label{lem:intersection}
  Let $x$ be a permutation such that $C(x)$ is not
  intersection-free. Then there exists an inversion that removes a
  pair of intersecting edges and replaces them with a pair of
  non-intersecting edges.  
\end{lemma}
\begin{proof}
  Suppose $\{x_{i-1},x_i\}$ and $\{x_j,x_{j+1}\}$ intersect in
  $C(x)$. Let $y = \inv{i}{j}[x]$. Then
  \begin{align*}
  C(x) \setminus C(y) &= \left\{\{x_{i-1},x_i\},\{x_j,x_{j+1}\}\right\}, 
  &\text{and}\\
  C(y) \setminus C(x) &= \left\{\{x_{i-1},x_j\},\{x_i,x_{j+1}\}\right\}.
\end{align*}
  By Proposition~\ref{prp:intersection}, since $\{x_{i-1},x_i\}$ and
  $\{x_j,x_{j+1}\}$ intersect, the two new edges introduced to $C(y)$
  by $\inv{i}{j}$ do not intersect.  Note that it is still
  possible that the introduced edges intersect with some of the
  remaining edges in $C(y)$.
\end{proof}

We denote by $\hull{V} \subseteq V$ the set of points in $V$ that
appear on the convex hull of $V$. A permutation $x$ respects
hull-order if any two points in the subsequence of $x$ induced by
$\hull{V}$ are consecutive in $x$ if and only if they are consecutive
on the hull.

\begin{lemma}
  \label{lem:hull-order}
  If $C(x)$ is intersection-free, then $x$ respects hull-order.
\end{lemma}
\begin{proof}
  This follows immediately from the proof of Theorem 2
  in~\cite{Quintas1965properties}.
\end{proof}
Lemma~\ref{lem:hull-order} entails the following bound on the number
of distinct intersection-free tours.
\begin{lemma}
  \label{lem:number-of-intersection-free-tours}
  Suppose $|V \setminus \hull{V}| = k$.  Then there are at most
  $\binom{n}{k} k!$ distinct intersection-free tours.
\end{lemma}
\begin{proof}
  We first claim the number of permutations on $n$ points in which
  some subset of $p < n$ points all remain in the same fixed order is
  $\binom{n}{n-p} (n-p)!$.  To construct such a permutation, each of
  the $n-p$ ``free'' points can be placed before all $p$ points, or
  after any of the $p$ points. Thus, for each of the $n-p$ free
  points, there are $p+1$ positions relative to the non-free points
  from which we can choose (with replacement) yielding
  $\binom{(p+1)+(n-p)-1}{n-p} = \binom{n}{n-p}$ choices for relative
  placement. For each such choice, there are $(n-p)!$ distinct ways to
  order the free points.

  Fix some point $v \in \hull{V}$. Suppose $C(x)$ is
  intersection-free. There is a transformation (specifically, a cyclic
  permutation) that maps $x$ to a sequence $z$ in which $v$ appears to
  the left of any other element in $\hull{V}$, yet $C(x) =
  C(z)$. Hence the number of distinct intersection-free tours is
  bounded by the number of sequences of points in $V$ where the
  elements of $\hull{V}$ appear in hull-order, and $v$ appears to the
  left of any other element of $\hull{V}$.  By the above claim, there
  are exactly $\binom{n}{k} k!$ such permutations and the lemma is
  proved.
\end{proof}

We also can derive from Lemma~\ref{lem:hull-order} the following
convenient bound on the minimal number of operations necessary to
transform an intersection-free tour into a permutation that
corresponds to a globally optimal tour.
\begin{lemma}
  \label{lem:jumps-to-solve-intersection-free}
  Suppose $|V \setminus \hull{V}| = k$ and $C(x)$ is an
  intersection-free tour on $V$. Then there is a sequence of at most
  $k$ jump operations that transforms $x$ into an optimal permutation.
\end{lemma}
\begin{proof}
  By Lemma~\ref{lem:hull-order}, since $C(x)$ is intersection-free,
  $x$ respects hull-order. Let $x^\star$ be an optimal permutation
  such that the elements in $\hull{V}$ have the same linear order in
  $x^\star$ as they do in $x$.  Then $x$ can be transformed into
  $x^\star$ by moving each of the $k$ inner points into their correct
  position.
\end{proof}

\begin{lemma}
  \label{lem:inversions-to-solve-intersection-free}
  Suppose $|V \setminus \hull{V}| = k$ and $C(x)$ is an
  intersection-free tour on $V$. Then there is a sequence of at most
  $2k$ inversions that transforms $x$ into an optimal permutation.
\end{lemma}
\begin{proof}  
  By Lemma~\ref{lem:jumps-to-solve-intersection-free}, $k$ jump
  operations are enough to transform $x$ into an optimal
  permutation. The claim then immediately follows from the fact that 
  any jump operation can be simulated by at most two consecutive
  inversion operations. In particular,
  \[
  \jmp{i}{j}[x] =
  \begin{cases}
    \inv{i}{j}[x] & \text{if $|i-j| = 1$;}\\
    \inv{i}{(j-1)}[\inv{i}{j}[x]] & \text{if $i - j < 1$;}\\
    \inv{(j+1)}{i}[\inv{j}{i}[x]] & \text{if $i - j > 1$.}
  \end{cases}
  \]
  Since a sequence of at most $k$ jump operations can be simulated by
  a sequence of at most $2k$ inversion operations, the claim is
  proved.
\end{proof}

A challenge to the runtime analysis of algorithms that employ edge
exchange operations such as 2-opt is that, when points are allowed in
arbitrary positions, the minimum change in fitness between neighboring
solutions can be made arbitrarily small. Indeed, proof techniques for
worst-case analysis often leverage this
fact~\cite{Englert2007worst}. To circumvent this, we impose bounds on
the angles between points, which allows us to express runtime results
as a function of trigonometric expressions involving these
bounds. Momentarily, we will refine this further by introducing a
class of TSP instances embedded in an $m \times m$ grid. In that case,
we will see that the resulting trigonometric expression is bounded by
a polynomial in $m$.

We say $V$ is \emph{angle-bounded by $\epsilon > 0$} if for any three
points $u,v,w \in V$, $0 < \epsilon < \theta < \pi-\epsilon$ where
$\theta$ denotes the angle formed by the line from $u$ to $v$ and the
line from $v$ to $w$. This allows us to express a bound in terms of
$\epsilon$ on the change in fitness from a move that removes an
inversion.

\begin{lemma}
  \label{lem:improvement}
  Suppose $V$ is angle-bounded by $\epsilon$. Let $x$ be a permutation
  such that $C(x)$ is not intersection-free. Let $y = \inv{i}{j}[x]$
  be the permutation constructed from an inversion on $x$ that
  replaces two intersecting edges in $C(x)$ with two non-intersecting
  edges.\footnote{Lemma~\ref{lem:intersection} guarantees the
    existence of such an inversion.} Then, if $d_{min}$ denotes the
  minimum distance between any two points in $V$, $f(x) - f(y) > 2
  d_{min} \left(\frac{1 - \cos(\epsilon)}{\cos(\epsilon)}\right)$.
\end{lemma}
\begin{proof}
  The inversion $\inv{i}{j}$ removes intersecting edges $\{u,v\}$ and
  $\{s,t\}$ from $C(x)$ and replaces them with the pair $\{s,u\}$ and
  $\{t,v\}$ to form $C(y)$. We label the point at which the original
  edges intersect as $p$.

  Denote as $\theta_u$ and $\theta_v$ the angles between the line
  segments that join at each point $u$ and $v$, respectively.  Since
  all angles are strictly positive, the points $u$, $s$, and $p$ form
  a nondegenerate triangle with angles $\theta_s$, $\theta_u$, and
  $(\pi - (\theta_s + \theta_u))$. By the law of sines we have
  \[
  \frac{d(s,u)}{\sin\left(\pi - (\theta_s + \theta_u)\right)} =
  \frac{d(s,u)}{\sin\left(\theta_s + \theta_u\right)} =
  \frac{d(u,p)}{\sin(\theta_s)} =
  \frac{d(s,p)}{\sin(\theta_u)}.  
  \]  
  Hence,
  \begin{equation}
    \label{eq:triangle1}
    d(u,p) + d(s,p) = d(s,u) \left(\frac{\sin (\theta_s) + \sin (\theta_u)}
      {\sin \left(\theta_s + \theta_u\right)}\right).
  \end{equation}
  Since $u$, $s$, and $p$ form a triangle, $0 < (\theta_s + \theta_u) < \pi$ and we have
  \begin{align*}
    0 &< \sin (\theta_s) < 1 & \text{since~} 0 &< \theta_s < \pi,\\
    0 &< \sin (\theta_u) < 1 & \text{since~} 0 &< \theta_u < \pi,\\
    0 &< \sin \left( \theta_s + \theta_u \right) < 1 & \text{since~} 0 &< \theta_s + \theta_u < \pi.
  \end{align*}

  Furthermore, since $V$ is angle-bounded by $0 < \epsilon < \pi -
  \epsilon$, by~(\ref{eq:triangle1}),
  \begin{equation}
    \label{eq:triangle2}
    d(u,p) + d(s,p) > d(s,u)\left(\frac{\sin (\epsilon) + \sin (\epsilon)}{\sin (\epsilon + \epsilon)}\right) > d(s,u).
  \end{equation}
  Since there is also a nondegenerate triangle formed by the points
  $t$, $v$, and $p$, a symmetric argument holds and thus
  \begin{equation}
    \label{eq:triangle3}
    d(t,p) + d(v,p) > d(t,v)\left(\frac{\sin (\epsilon) + \sin (\epsilon)}{\sin (\epsilon + \epsilon)}\right) > d(t,v).
  \end{equation}  
  Combining Equations~(\ref{eq:triangle2})
  and~(\ref{eq:triangle3}) we have
  \begin{align*}
   & f(x) - f(y) = \left[ d(u,v) + d(s,t) \right] - \left[ d(t,v) + d(s,u) \right]\\
    &= d(u,p) + d(v,p) + d(t,p) + d(s,p) 
     - \left[ d(t,v) + d(s,u) \right]\\
    &> \left[d(t,v) + d(s,u)\right]\left(\frac{2\sin (\epsilon)}{\sin (2\epsilon)}\right)
     - \left[ d(t,v) + d(s,u) \right] > 0  
  \end{align*}
  The constraint that the difference is strictly positive follows
  directly from Equations (\ref{eq:triangle2}) and
  (\ref{eq:triangle3}). Hence,
    \begin{align*}
      f(x) - f(y) > \left[d(t,v) + d(s,u)\right]\left(\frac{2 \sin (\epsilon)}{\sin (2 \epsilon)} - 1\right)\\
      \geq 2 d_{min} \left(\frac{2 \sin (\epsilon)}{\sin (2
          \epsilon)} - 1\right)
       = 2 d_{min} \left(\frac{1 - \cos(\epsilon)}{\cos(\epsilon)}\right). &\qedhere
    \end{align*}
\end{proof}

The lower bound on the angle between any three points in $V$ provides
a constraint on how small the change in fitness between neighboring
inversions can be. This lower bound is useful in the case of a
\emph{quantized} point-set. That is, when the points can be embedded
on an $m \times m$ grid as illustrated in
Figure~\ref{fig:quantization}.

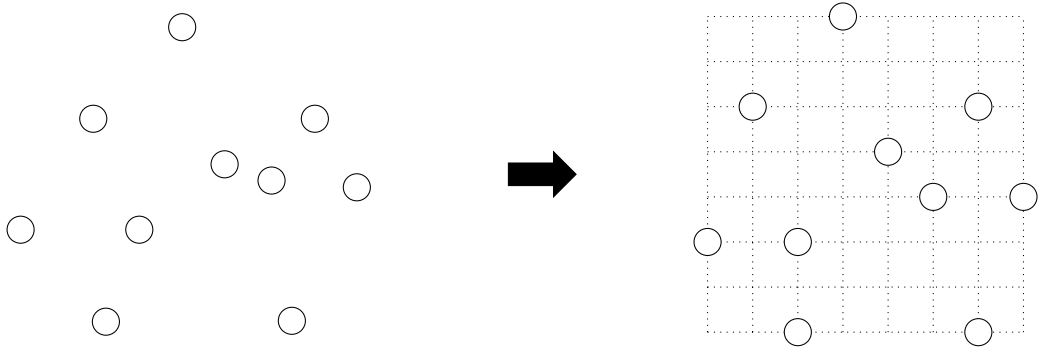
\begin{figure}
  \centering
  \begin{tikzpicture}
    \node (unquantized) at (0,0) {%
      \begin{tikzpicture}[scale=0.6]
        \def\xa{-0.25}\def\ya{0.27}
        \def\xb{0.36}\def\yb{-0.27}
        \def\xc{-0.36}\def\yc{0.23}
        \def\xd{0.38}\def\yd{0.27}
        \def\xe{0.33}\def\ye{-0.24}
        \def\xf{0.27}\def\yf{-0.28}
        \def\xg{0.31}\def\yg{0.36}
        \def\xh{-0.24}\def\yh{0.25}
        \def\xi{0.27}\def\yi{-0.27}
        \def\xj{0.20}\def\yj{0.21}
   
        \draw[fill=white] (0+\xa,2+\ya) circle (0.3cm);
        \draw[fill=white] (1+\xb,5+\yb) circle (0.3cm);
        \draw[fill=white] (2+\xc,0+\yc) circle (0.3cm);
        \draw[fill=white] (2+\xd,2+\yd) circle (0.3cm);
        \draw[fill=white] (3+\xe,7+\ye) circle (0.3cm);
        \draw[fill=white] (4+\xf,4+\yf) circle (0.3cm);
        \draw[fill=white] (5+\xg,3+\yg) circle (0.3cm);
        \draw[fill=white] (6+\xh,0+\yh) circle (0.3cm);   
        \draw[fill=white] (6+\xi,5+\yi) circle (0.3cm);
        \draw[fill=white] (7+\xj,3+\yj) circle (0.3cm);
      \end{tikzpicture}
    };
    \node (arrow) at (4.7,0) {%
      \begin{tikzpicture}[scale=0.3]
        \draw[fill=black] (0,0) -- (0,1) -- (2,1) -- (2,1.5) -- (3,0.5) --
        (2,-0.5) -- (2,0) -- cycle;
      \end{tikzpicture}
    };
    \node (quantized) at (9,0){%
      \begin{tikzpicture}[scale=0.6]
        \draw[step=1cm,dotted] (0,0) grid (7,7);
        \draw[fill=white] (0,2) circle (0.3cm);
        \draw[fill=white] (1,5) circle (0.3cm);
        \draw[fill=white] (2,0) circle (0.3cm);
        \draw[fill=white] (2,2) circle (0.3cm);
        \draw[fill=white] (3,7) circle (0.3cm);
        \draw[fill=white] (4,4) circle (0.3cm);
        \draw[fill=white] (5,3) circle (0.3cm);
        \draw[fill=white] (6,0) circle (0.3cm);   
        \draw[fill=white] (6,5) circle (0.3cm);
        \draw[fill=white] (7,3) circle (0.3cm);
      \end{tikzpicture}
    };
  \end{tikzpicture}

  \caption{\label{fig:quantization} Quantizing a set of $n$ points in
    the plane onto an $m \times m$ grid.}
\end{figure}

Quantization, for example, occurs when the $x$ and $y$ coordinates of
each point in the set are rounded to the nearest value in a set of $m$
equidistant values (e.g., integers). We point out that it is still
important that the quantization preserves the constraint on
collinearity since collinear points violate a nonzero angle bound. We
have the following lemma.

\begin{lemma}
  \label{lem:grid-angle-bound}
  Suppose $V$ is a set of points that lie on an $m \times m$ unit
  grid, no three collinear. Then $V$ is angle-bounded by $\arctan
  \left(1/(2(m-2)^2)\right)$.
\end{lemma}
\begin{proof}
  The grid imposes a coordinate system on $V$ in which the concept of
  line slope is well-defined.  Let $u,v,w \in V$ be arbitrary points.
  We consider the angle $\theta$ at point $v$ formed by the lines from
  $v$ to $u$ and $v$ to $w$. Let $s_1$ and $s_2$ denote the slope of
  these lines, respectively. If the slopes are of opposite sign, then
  $\theta \geq 2\arctan((m-1)^{-1})$ since the lines form hypotenuses
  of two right triangles with adjacent sides of length at most $m-1$
  and opposite sides with length at least $1$ (see
  Figure~\ref{fig:bisect}).
  
  We now consider the case where the slopes are nonnegative. The
  nonpositive case is handled identically (or by simply changing the
  sign of the slopes by the appropriate transformation). Without loss
  of generality, assume $s_1 > s_2 \geq 0$. Equality is impossible
  since $u$, $v$, and $w$ cannot be collinear. Since the points lie on
  an $m \times m$ grid, $s_1$ and $s_2$ must be ratios of whole
  numbers at most $m-1$, say $s_1 = a/b$ and $s_2 = c/d$. The angle at
  point $v$ is
  $
  \theta = \arctan(a/b) - \arctan(c/d) = 
  \arctan \left( \frac{ad - cb}{bd + ac} \right).
  $
  The minimum positive value for the expression $(ad - cb)/(bd + ac)$
  over the integers from $0$ to $m-1$ is $\frac{1}{2(m-2)^2}$.  Since
  the inverse of the tangent is monotone, the minimum nonzero angle
  must be $\theta \geq \arctan \left(1/(2(m-2)^2) \right)$.
\end{proof}

Lemma~\ref{lem:grid-angle-bound} allows us to translate the somewhat
awkward trigonometric expression in the claim of
Lemma~\ref{lem:improvement} (and subsequent lemmas that depend on it)
into a convenient polynomial that can be expressed in terms of $m$.

\begin{lemma}
  \label{lem:quantized}
  Let $V$ be a set of $n$ points that lie on an $m \times m$ unit
  grid, no three collinear. Then, $V$ is angle-bounded by $\epsilon$
  where $\cos(\epsilon)/(1-\cos(\epsilon)) = O(m^4)$.
\end{lemma}
\begin{proof}
  It follows from Lemma~\ref{lem:grid-angle-bound} that the angle
  bound on $V$ is $\epsilon = \arctan \left( 1/(2(m-2)^2) \right)$.
  Since $\cos (\arctan (x)) = 1/\sqrt{1+x^2}$ we have
  \[
  \frac{\cos(\epsilon)}{1-\cos(\epsilon)} = \frac{2(m-2)^2}{\sqrt{1+4(m-2)^4} - 2(m-2)^2}.
  \]
  and since $z/(\sqrt{1+z^2} - z) = O(z^2)$ , setting $z = 2(m-2)^2$
  completes the proof.
\end{proof}

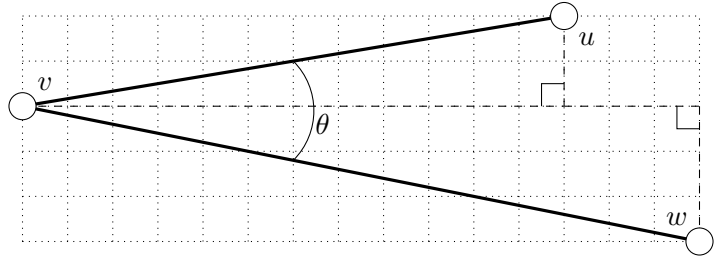
\begin{figure}
  \centering
  \begin{tikzpicture}[scale=0.6]
    \draw[step=1cm,dotted] (0,0) grid (15,5);

    \draw[very thick] (0,3) edge (12,5); 
    \draw[very thick] (0,3) edge (15,0); 

    \draw[dashed] (0,3) edge (12,3);
    \draw[dashed] (12,3) edge (12,5);
    \draw (11.5,3) -- (11.5,3.5) -- (12,3.5);

    \draw[dashed] (12,3) edge (15,3);
    \draw[dashed] (15,3) edge (15,0);
    \draw (14.5,3) -- (14.5,2.5) -- (15,2.5);

    \draw[fill=white] (0,3) circle (0.3cm);
    \draw[fill=white] (12,5) circle (0.3cm);
    \draw[fill=white] (15,0) circle (0.3cm);

    \def\x{6}    
    \def\y{4}    
    \def\z{1.8}  
    \def\r{1.555635}   
    \draw (\x,\y) arc (45:-45:\r);

    \node[draw=none] at (6.65,2.6) {$\theta$};
    \node[draw=none] at (0.5,3.5) {$v$};
    \node[draw=none] at (12.5,4.5) {$u$};
    \node[draw=none] at (14.5,0.5) {$w$};
  \end{tikzpicture}
  
  \caption{\label{fig:bisect} If the slope of the lines from $v$ to
    $u$ and $v$ to $w$ are of opposite sign, they form the hypotenuses
    of two right triangles and $\theta \geq 2\arctan((m-1)^{-1})$.}
\end{figure}

We are now ready to prove the following technical lemma for 2-opt
mutation defined in Function~\ref{fun:2-opt-mutation}. This lemma will
be instrumental in proving runtime bounds later in the paper.
\begin{lemma}
  \label{lem:2-opt-mutation}
  Let $V$ be a set of planar points in convex position angle-bounded
  by $\epsilon$. We have the following.
  \begin{enumerate}
  \item[(1)] For any $x \in S_{\alpha}$, the probability that 2-opt
    mutation creates an offspring $y$ with $f(y) < f(x) - 2 d_{min}
    \left(1-\cos(\epsilon)\right)/\left(\cos(\epsilon)\right)$
    is at least $2/(\e n(n-1))$.
  \item[(2)] For any $x \in S_{\beta}$, the probability that 2-opt mutation
    creates an optimal solution is at least $(\e n^{4k} (2k-1)!)^{-1}$
  \end{enumerate}
\end{lemma}
\begin{proof}
  For (1), suppose $x \in S_{\alpha}$.  By
  Lemma~\ref{lem:intersection}, there is at least one pair of
  intersecting edges in $C(x)$ that can be removed with a single
  inversion operation $\inv{i}{j}$.  Let $y = \inv{i}{j}[x]$. By
  Lemma~\ref{lem:improvement}, $f(y)$ satisfies the fitness bound in
  the claim. It suffices to bound the probability that $y$ is produced
  by \FuncSty{2-opt-mutation($x$)}.

  Let $E_1$ denote the event that Poisson mutation performs exactly
  one inversion (i.e., $s = 0$ in line~\ref{li:Poisson} of
  Function~\ref{fun:2-opt-mutation}).  Let $E_2$ denote the event that
  the pair $(i,j)$ is specifically chosen for the inversion. 

  We have $\Pr\{E_1\} = \e^{-1}$ from the Poisson density function and
  $\Pr\{E_2|E_1\} \geq (n(n-1)/2)^{-1}$. Thus the probability that
  2-opt mutation creates $y$ from $x$ is
  \[
  \Pr\{E_1 \cap E_2\} = \Pr\{E_1\}\cdot\Pr\{E_2|E_1\} \geq 2/(\e n(n-1)).
  \]

  For (2), suppose $x \in S_{\beta}$. Thus, $C(x)$ is
  intersection-free, and it follows from
  Lemma~\ref{lem:inversions-to-solve-intersection-free} that there are
  at most $2k$ inversion moves that transform $x$ into an optimal
  solution.

  Let $E'_1$ denote the event that Poisson mutation performs exactly
  $2k$ inversions (i.e., $s = 2k -1$ in line~\ref{li:Poisson} of
  Function~\ref{fun:2-opt-mutation}).  Let $E'_2$ denote the event that
  all $2k$ inversions are the correct moves that transform $x$
  into an optimal solution. 

  Again, from the Poisson density function, $\Pr\{E'_1\} =
  (\e(2k-1)!)^{-1}$. Since $\Pr\{E'_2|E'_1\} \geq (n(n-1)/2)^{-2k} \geq
  n^{-4k}$, the probability of transforming $x$ into an optimal
  solution is at least
  \[
  \Pr\{E'_1 \cap E'_2\} =
  \Pr\{E'_1\}\cdot\Pr\{E'_2|E'_1\} \geq (\e n^{4k}(2k-1)!)^{-1}.
  \]
\end{proof}

Finally, since the time bounds in the remainder of this paper are
expressed as a function of the angle bound $\epsilon$, and the bounds
on point distance, it will be useful to define the following function.

\begin{definition}\label{def:fn}
  Let $V$ be a set of points angle-bounded by $\epsilon$. We define
  \[  
  \fn(\epsilon) = \left(\frac{d_{max}}{d_{min}} -
    1\right)\left(\frac{\cos(\epsilon)}{1-\cos(\epsilon)}\right)
  \]
  where $d_{max}$ and $d_{min}$ respectively denote the maximum and
  minimum Euclidean distance between points in $V$.
\end{definition}

\section{Instances in convex position}\label{sec:inst-conv-posit}
A finite point set $V$ is in \emph{convex position} when every point
in $V$ is a vertex on its convex hull. De{\u{\i}}neko et
al.~\cite{Deineko2006inner} observed that the Euclidean TSP is easy to
solve when $V$ is in convex position.  In this case, the optimal
permutation is any linear ordering of the points which respects the
ordering of the points on the convex hull. Such an ordering can be
found in time $O(n \log n)$~\cite{deBerg2008computationalgeometry}.

In the context of evolutionary algorithms, the natural question
arises, if $V$ is in convex position, how easy is it for simple
randomized search heuristics?  In this case, a tour is
intersection-free if and only if it is globally optimal, hence finding
an optimal solution is exactly as hard as finding an intersection-free
tour. Since an intersection can be (at least temporarily) removed from
a tour by an inversion operation (c.f. Lemma~\ref{lem:intersection}), we
focus in this section on algorithms that use the inversion operation.

\subsection{RLS}\label{sec:rls}
RLS operates on a single candidate solution by performing a single
inversion in each iteration. Since each inversion which removes an
intersection results in a permutation whose fitness is improved by the
amount bounded in Lemma~\ref{lem:improvement}, it is now
straightforward to bound the time it takes for RLS to discover a
permutation that corresponds to an intersection-free tour. 

\begin{theorem}
  \label{thm:convex-position-rls}
  Let $V$ be a set of planar points in convex position angle-bounded
  by $\epsilon$. The expected time for RLS to solve the TSP on $V$ is
  $O(n^3 \fn(\epsilon))$ where $\fn$ is defined in
  Definition~\ref{def:fn}.
 \end{theorem}
  \begin{proof}
    Consider the infinite stochastic process generated by
    Algorithm~\ref{alg:rls} $(x^{(1)},x^{(2)},\ldots)$.  It suffices
    to bound the expectation of the random variable $T$ defined in
    Equation~(\ref{eq:T}).  Since $V$ is in convex position, any
    intersection-free tour is globally optimal. Thus, in this case
    $S_{\beta} = \emptyset$, and Equation (\ref{eq:alphabeta}) can be
    written as
    \[
    T = \sum_{t=1}^\infty \alpha(x^{(t)}).
    \]    
    Consider an arbitrary permutation $x \in S_{\alpha}$. Since $C(x)$
    must contain intersections, by Lemma~\ref{lem:intersection}, there
    is an inversion $\inv{i}{j}$ which removes a pair of intersecting
    edges and replaces them with a pair of non-intersecting edges.
    Moreover, by Lemma~\ref{lem:improvement}, such an inversion
    results in an improvement of at least
    \begin{equation}
      2 d_{min} \left(1-\cos(\epsilon)\right)/\left(\cos(\epsilon)\right).\label{eq:rls-c}    
    \end{equation}
    If the state $x$ is visited by RLS, the probability that this
    particular inversion is selected uniformly at random is
    $2/(n(n-1))$. Thus we have the conditions of Lemma~\ref{lem:alpha}
    with $p = 2/(n(n-1))$ and $c$ equal to the expression
    in~(\ref{eq:rls-c}) and the claimed bound on the expectation of
    $T$ follows.
\end{proof}

\begin{corollary}[to Theorem~\ref{thm:convex-position-rls}]
  \label{cor:convex-position-rls}
  If $V$ is in convex position and embedded in an $m \times m$ grid
  with no three points collinear, then RLS solves the TSP on $V$ in
  expected time $O(n^3 m^5)$.
\end{corollary}
\begin{proof}
  The bound follows immediately from
  Theorem~\ref{thm:convex-position-rls} since $V$ is in convex
  position and since, by Lemma~\ref{lem:grid-angle-bound}, $V$ is
  angle-bounded by $\arctan \left( 1/(2(m-2)^2) \right)$ , $d_{max} =
  (m-1)\sqrt{2}$, and $d_{min} = 1$. Appealing to
  Lemma~\ref{lem:quantized} yields $\cos(\epsilon)/(1-\cos(\epsilon))
  = O(m^4)$. Substituting these terms into bound of
  Theorem~\ref{thm:convex-position-rls} completes the proof.
\end{proof}

\subsection{The $\boldsymbol{(\mu+\lambda)}$~EA}\label{sec:ea}
We now consider the optimization time of the $(\mu+\lambda)$~EA using
2-opt mutation defined in Algorithm~\ref{alg:ea} applied to a set of
points in convex position. Obviously, in this case we will find $\mu$
and $\lambda$ terms appearing in the runtime formulas. We will assume
that $\mu$ and $\lambda$ are polynomials in both $n$ and $k$.  We will
also find that setting $\lambda = \Theta(\mu n^2)$ ensures that a
transition from any state $x^{(t)} \in S_{\alpha}$ a state which
improves on the fitness by at least a specified amount occurs with
constant probability.

Such a setting has the effect of reducing the number of
\emph{generations} spent removing intersections from the best-so-far
tours. However, it is important to note that the number of calls to
the fitness function must be accordingly increased by a factor of
$\Theta(\mu n^2)$ in each generation. Nevertheless, the expected
number of generations can be a useful measure when considering
parallel evolutionary algorithms. As Jansen, De Jong and
Wegener~\cite{Jansen2005choice} have pointed out, when the fitness
evaluation of the offspring can be performed on parallel processors,
the number of generations corresponds to the \emph{parallel}
optimization time. In such a case, we would observe a quadratic factor
improvement in the parallel runtime corresponding to the segment spend
in tours with intersections.

The following theorem bounds the number of expected generations the
$(\mu+\lambda)$~EA needs to solve the Euclidean TSP on a set of
angle-bounded points in convex position.

\begin{theorem}
  \label{thm:convex-position-mulambda}
  Let $V$ be a set of planar points in convex position angle-bounded
  by $\epsilon$. The expected time for the $(\mu+\lambda)$~EA using
  2-opt mutation to solve the TSP on $V$ is bounded above by
  $O{\left((\mu/\lambda)\cdot n^3 \fn(\epsilon) +
      n\fn(\epsilon)\right)}$ where $\fn$ is as defined in
  Definition~\ref{def:fn}.
 \end{theorem}
 \begin{proof}
   The sequence of best-so-far permutations generated by
   Algorithm~\ref{alg:ea} is the infinite stochastic process
   $(x^{(1)},x^{(2)},\ldots)$ and we seek the expectation of the
   random variable $T$ defined in Equation~(\ref{eq:T}).  Again, in
   the case of convex position, there are no non-optimal
   intersection-free tours so that $S_{\beta} = \emptyset$ and
   \[
   T = \sum_{t=1}^\infty \alpha(x^{(t)}).
   \]
   As long as $x^{(t)}$ is non-optimal, $C(x^{(t)})$ must contain at
   least one pair of intersecting edges.  Hence, in generation $t$, if
   $x^{(t)}$ is selected for mutation to create one of the $\lambda$
   offspring, then by Lemma~\ref{lem:2-opt-mutation}, 2-opt mutation
   must improve the best-so-far solution by at least $c = 2 d_{min}
   \left(1-\cos(\epsilon)\right)/\left(\cos(\epsilon)\right)$ with
   probability at least $(\e n(n-1)/2)^{-1}$. The probability that at
   least one of the $\lambda$ offspring improves on $x^{(t)}$ by at
   least this amount is
   \[
   p \geq 1 - \left(1 - \frac{1}{\mu\e n(n-1)/2}\right)^\lambda.
   \]
   We now make the following case distinction on $\lambda$.
   \begin{description}
   \item[Case $\lambda \geq \mu\e n (n-1)/ 2$.] For this setting of
     $\lambda$, we have $p \geq 1 - \e^{-1}$, so an intersection is
     removed in each generation with constant probability. Invoking
     Lemma~\ref{lem:alpha}, the expected time to find an
     intersection-free tour is at most $O(n\fn(\epsilon))$.
   \item[Case $\lambda < \mu\e n (n-1)/2$.] Here we have
     \[
     \left(1 - \frac{1}{\mu\e n(n-1)/2}\right)^\lambda \geq 
     1 - \e^{-\lambda/(\mu\e n(n-1)/2)} \geq
     \frac{\lambda}{\mu\e n(n-1)}.
     \]
     The final inequality comes from the fact that $\e^{-x} \leq 1 -
     x/2$ for $0 \leq x \leq 1$. Thus, in this case we have $p \geq
     \lambda/(\mu \e n^2)$.  Applying Lemma~\ref{lem:alpha}, the
     expected time to find an intersection-free tour is at most
     $O((\mu/\lambda)\cdot n^3\fn(\epsilon))$.
   \end{description}
   Thus the expected time to solve the instance is bounded by
   $O(\max\{(\mu/\lambda)\cdot n^3 \fn(\epsilon), n \fn(\epsilon)\})$
   which yields the claimed bound.   
 \end{proof}

 Analogous to the Corollary to Theorem~\ref{thm:convex-position-rls},
 we have the following.

\begin{corollary}[to Theorem~\ref{thm:convex-position-mulambda}]
  \label{cor:convex-position-mulambda}
  If $V$ is in convex position and embedded in an $m \times m$ grid
  with no three points collinear, then RLS solves the TSP on $V$ in
  expected time $O((\mu/\lambda) \cdot n^3 m^5 + n m^5)$.
\end{corollary}

\section{Parameterized runtime analysis}\label{sec:param-runt-analys}
We now turn our attention to TSP instances in which $\hull{V} \neq V$
and express the expected runtime in terms of $n = |V|$, the number of
points, and $k = |V\setminus \hull{V}|$, the number of inner
points.

\subsection{Expected time to find 2-opt local optima}
When $\hull{V} \neq V$, RLS does not necessarily converge to the
global optimum with probability one since it can become trapped in a
local optimum.  However, if the number of inner points is sparse
(i.e., $O(1)$), we can bound the complexity of finding local optima.
In other words, we want to estimate $E(T_{loc})$, defined in
Equation~(\ref{eq:Tloc}).

For RLS, a permutation $x$ is \emph{locally optimum} in the 2-opt
neighborhood if there does not exist an inversion $y = \inv{i}{j}[x]$
such that $f(y) < f(x)$ for any $1 \leq i < j \leq n$. In this case,
RLS cannot make further improvements. Indeed, if no three points are
collinear in $V$, then if $x$ is locally optimal, $C(x)$ must be
intersection-free since, if it were not, Lemma~\ref{lem:intersection}
guarantees an inversion exists that removes an intersection and, by
the quadrangle inequality, the resulting offspring would have
improving fitness. It is important to note that the converse is not
necessarily true. However, we may show the following.

\begin{lemma}
  \label{lem:intersection-free-neighbor}
  Suppose $C(x)$ is intersection-free. Then for any neighboring
  inversion $y = \inv{i}{j}[x]$ with $f(y) < f(x)$, $C(y)$ is also
  intersection-free.
\end{lemma}
\begin{proof}
  Suppose for contradiction that $y = \inv{i}{j}[x]$ with $f(y) < f(x)$,
  but $C(y)$ is not intersection-free. Since $C(x)$ was
  intersection-free, it follows that the pair of edges introduced by
  the inversion must intersect. By the quadrangle inequality, the
  total length of these edges must be greater than the edges they
  replaced, contradicting that $f(y)$ is strictly less than $f(x)$.
\end{proof}

\begin{theorem}
  \label{thm:rls-local-optimum}
  Suppose $V$ is angle bounded by $\epsilon$ and that $|V \setminus
  \hull{V}| = k$. Then the expected time until RLS finds a local
  optimum in the 2-opt neighborhood is $O{\left(n^3\fn(\epsilon) +
      n^{k+2}k!\right)}$ where $\fn$ is as defined in 
  Definition~\ref{def:fn}.
\end{theorem}
\begin{proof}
  After RLS finds an intersection-free tour, by
  Lemma~\ref{lem:intersection-free-neighbor}, all subsequent tours
  will also be intersection free.  The total number of distinct
  intersection-free tours hence serves as a bound on the number of
  possible improving moves after the first intersection-free tour is
  encountered. By Lemma~\ref{lem:number-of-intersection-free-tours},
  this bound is $\binom{n}{k} k! = O(n^{k}k!)$.

  As long as RLS has not yet found a local optimum, by definition
  there exists a strictly improving inversion; the expected waiting
  time to find such an inversion is bounded by $O(n^2)$. Thus, after
  the first time an intersection-free tour is encountered, the
  expected time until a local optimum is found is bounded by
  $O(n^{k+2}k!)$.

  Finally, for any $x^{(t)} \in S_{\alpha}$, by
  Lemma~\ref{lem:improvement} there is a pair $(i,j)$ such that
  \[
  f(\inv{i}{j}[x^{(t)}]) < f(x^{(t)}) - 2 d_{min} \left(\frac{1 -
      \cos(\epsilon)}{\cos(\epsilon)}\right).
  \]
  The probability that RLS selects this inversion is
  $2/(n(n-1))$. Substituting these values into Lemma~\ref{lem:alpha}
  completes the proof.
\end{proof}

If the points in $V$ are quantized in an $m \times m$ grid, we can
appeal directly to Lemmas~\ref{lem:grid-angle-bound}
and~\ref{lem:quantized} to substitute the corresponding angle bounds
into the bound obtained in Theorem~\ref{thm:rls-local-optimum}. This
results in the following corollary.
 
\begin{corollary}[to Theorem~\ref{thm:rls-local-optimum}]
  \label{cor:rls-local-optimum}
  Suppose that $V$ is quantized in an $m \times m$ grid and that $|V
  \setminus \hull{V}| = k$. Then the expected time until RLS finds a
  local optimum is $O{\left(n^3m^5\right)} +
  O{\left(n^{k+2}k!\right)}$.
\end{corollary}

\subsection{The $\boldsymbol{(\mu+\lambda)}$~EA using 2-opt mutation}

The $(\mu+\lambda)$~EA does not suffer from convergence to local
optima as does RLS since it uses a Poisson mutation strategy and thus
has a non-zero probability of generating any tour (this follows from
the connectedness of the inversion adjacency, see, e.g.,
\cite{Hagita2002diameters}).

We now use the structural analysis in Section~\ref{sec:struct-prop} to
show that when there are few inner points, intersection-free tours are
somehow ``close'' to an optimal solution in the sense that relatively
small perturbations by the EA suffice to solve the problem.  Again,
recall that we assume $\mu$ and $\lambda$ are polynomials in both $n$
and $k$.

\begin{theorem}
   \label{thm:ea-2opt}
   Let $V$ be a set of points angle-bounded by $\epsilon$ such that
   $|V \setminus \hull{V}| = k$.  The expected time for the
   $(\mu+\lambda)$~EA using 2-opt mutation to solve the TSP on $V$ is
   bounded above by $O\left((\mu/\lambda) \cdot n^3\fn(\epsilon) +
     n\fn(\epsilon) + (\mu/\lambda) \cdot n^{4k}(2k-1)!\right)$.
 \end{theorem}
\begin{proof}   
  We argue by analyzing the Markov chain generated by the
  $(\mu+\lambda)$~EA using 2-opt mutation.  Let
  $(x^{(1)},x^{(2)},\ldots)$ denote the sequence of best-so-far states
  visited by the $(\mu+\lambda)$~EA\@. 

  In generation $t$, if $x^{(t)} \in S_{\alpha}$, then the probability
  of generating an offspring that improves the fitness by at least $c
  = 2 d_{min}
  \left(1-\cos(\epsilon)\right)/\left(\cos(\epsilon)\right)$ is
  identical to that in the proof of
  Theorem~\ref{thm:convex-position-mulambda}. Arguing in the same
  manner as in the proof of
  Theorem~\ref{thm:convex-position-mulambda}, we have
  \[
  \sum_{t=1}^\infty \alpha(x^{(t)}) = O((\mu/\lambda)\cdot
  n^3\fn(\epsilon)) + n\fn(\epsilon)).
  \]

  On the other hand, if $x^{(t)} \in S_{\beta}$, if it is selected for
  mutation, then, by Lemma~\ref{lem:2-opt-mutation}, 2-opt mutation
  produces the optimal solution with probability at least $(\e n^{4k}
  (2k-1)!)^{-1}$. Hence, the overall probability of generating an
  optimal permutation when $x^{(t)}$ is the (intersection-free)
  population-best permutation is at least
  \[
  q \geq 1 - \left(1 - \frac{1}{\mu \e n^{4k} (2k-1)!
    }\right)^{\lambda}
  \geq \frac{\lambda}{2 \mu \e n^{4k} (2k-1)!}.
  \]
  Since this is the probability that the Markov chain transits from a
  state $x^{(t)} \in S_{\beta}$ to the optimal state, substituting the
  value for $q$ into the claim of Lemma~\ref{lem:beta}, we have
  \[
  \E\left(\sum_{t=1}^{\infty} \beta(x^{(t)})\right) = O( (\mu/\lambda) \cdot n^{4k} (2k-1)!).
  \]
  The bound on $\E(T)$ then follows from Equation~\ref{eq:alphabeta}
  and linearity of expectation.
\end{proof}

Again, from Lemmas~\ref{lem:grid-angle-bound} and~\ref{lem:quantized}
we have the following corollary.

\begin{corollary}[to Theorem~\ref{thm:ea-2opt}]
  \label{cor:ea-2opt}
  Let $V$ be a set of points quantized on an $m \times m$ grid such
  that $|V \setminus \hull{V}| = k$.  The expected time for the
  $(\mu+\lambda)$~EA using 2-opt mutation to solve the TSP on $V$ is
  $O( (\mu/\lambda) \cdot n^3 m^5 + n m^5 + (\mu/\lambda) \cdot n^{4k}(2k-1)!)$.
\end{corollary}

\subsection{Mixed mutation strategies}
\label{sec:mixed}

Our analysis so far has revealed important insights into the problem
structure of the Euclidean TSP with $k$ inner points and the inversion
operator. We now take advantage of these insights to design a new
evolutionary algorithm with the aim of explicitly reducing the bound
on the expected number of iterations until a permutation corresponding
to an optimal TSP tour is found. 

The Markov chain analysis relies on the inversion operator to
construct an intersection-free tour, but then relies on the inversion
operator to simulate the jump operator in order to transform an
intersection-free tour into an optimal solution. We now introduce a
mutation technique called \FuncSty{mixed-mutation} (outlined in
Function~\ref{fun:mixed-mutation}) that performs both inversion and
jump operations, each with constant probability. This allows for a
bound which is faster by a factor of $O(n^{2k} (2k-1)!/(k-1)!)$.

Similar to Lemma~\ref{lem:2-opt-mutation}, we have the following
result for mixed mutation.
\begin{lemma}
  \label{lem:mixed-mutation}
  Let $V$ be a set of planar points in convex position angle-bounded
  by $\epsilon$. Then,
  \begin{enumerate}
  \item[(1)] For any $x \in S_{\alpha}$, the probability that 
    mixed mutation creates an offspring $y$ with $f(y) < f(x) - 2 d_{min}
    \left(1-\cos(\epsilon)\right)/\left(\cos(\epsilon)\right)$
    is at least $(\e n(n-1))^{-1}$.
  \item[(2)] For any $x \in S_{\beta}$, the probability that mixed mutation
    creates an optimal solution is at least $(2\e n^{2k} (k-1)!)^{-1}$
  \end{enumerate}
\end{lemma}
\begin{proof}
  For (1), the probability that mixed mutation selects inversions is
  $1/2$ and the rest of the claim follows from the argument of
  Lemma~\ref{lem:2-opt-mutation}.

  For (2), suppose $x \in S_{\beta}$. Thus, $C(x)$ is
  intersection-free, and it follows from
  Lemma~\ref{lem:jumps-to-solve-intersection-free} that there are at
  most $k$ jump operations that transform $x$ into an optimal solution
  and the remainder of the proof is identical to that of
  Lemma~\ref{lem:2-opt-mutation}. The probability that mixed mutation
  selects jump operations contributes the leading factor of $2^{-1}$.
\end{proof}

\begin{function}
\SetKwInOut{Input}{input}
\SetKwInOut{Output}{output}
\SetProcNameSty{texttt}
\Input{A permutation $x$}
\Output{A permutation $y$}
\BlankLine
$y \gets x$\;
draw $r$ from a uniform distribution on the interval $[0,1]$\; 
draw $s$ from a Poisson distribution with parameter $1$\;
\If{$r < 1/2$}{perform $s+1$ random inversion operations on $y$\;}
\Else{perform $s+1$ random jump operations on $y$\;}
\Return $y$\;
\caption{mixed-mutation($x$)} \label{fun:mixed-mutation}
\end{function}

\begin{theorem}
  \label{thm:ea-mixed}
  Let $V$ be a set of points angle-bounded by $\epsilon$ such that $|V
  \setminus \hull{V}| = k$.  The expected time for the
  $(\mu+\lambda)$~EA using mixed mutation to solve the TSP on $V$ is
  bounded above by $O{\left( (\mu/\lambda) \cdot n^3\fn(\epsilon) +
      n\fn(\epsilon) + (\mu/\lambda) \cdot n^{2k}(k-1)!\right)}$.
\end{theorem}
\begin{proof}
  The proof is identical to the proof of Theorem~\ref{thm:ea-2opt},
  except we substitute the probabilities from
  Lemma~\ref{lem:mixed-mutation} into Lemmas~\ref{lem:alpha}
  and~\ref{lem:beta}.
\end{proof}

As before, from Lemmas~\ref{lem:grid-angle-bound}
and~\ref{lem:quantized} we have the following.

\begin{corollary}[to Theorem~\ref{thm:ea-mixed}]
  \label{cor:ea-mixed}
  Let $V$ be a set of points quantized on an $m \times m$ grid such
  that $|V \setminus \hull{V}| = k$.  The expected time for the
  $(\mu+\lambda)$~EA using mixed mutation to solve the TSP on $V$ is
  $O((\mu/\lambda) \cdot n^3 m^5 + n m^5 +(\mu/\lambda) \cdot n^{2k}(k-1)!)$.
\end{corollary}

\section{Conclusion}\label{sec:conclusion}
In this paper, we have studied the runtime complexity of evolutionary
algorithms on the Euclidean TSP\@. We have carried out a parameterized
analysis that studies the dependence of the hardness of a problem
instance on the number of inner points in the instance. Moreover,
we have shown that under reasonable geometric constraints (low angle
bounds), simple evolutionary algorithms solve the convex TSP in
polynomial time. 

In the case that an instance contains $k$ inner points, we have
shown that randomized local search using 2-opt can find local optima
in expected $O(n^3\fn(\epsilon)) + O(n^{2k}k!)$ iterations where
$\fn$ is a function of the angle bound $\epsilon$. For example, when
the instance is embedded in an $m \times m$ grid, $\fn(\epsilon) =
O(m^5)$.

Similarly, for the $(\mu + \lambda)$~EA, we have bounded the expected
number of generations to solve a TSP instance with $k$ inner points as
$O((\mu/\lambda) \cdot n^3 \fn(\epsilon) + n\fn(\epsilon)
+(\mu/\lambda) \cdot n^{4k}(2k-1)!)$. Using the analysis, we have also
introduced a mixed mutation strategy based on both permutation jumps
and 2-opt moves which attains an improved expected runtime bound of
$O((\mu/\lambda) \cdot n^3 \fn(\epsilon) + n\fn(\epsilon) +
(\mu/\lambda) \cdot n^{2k}(k-1)!)$. Hence, with this paper, we have
shown that the $(\mu+\lambda)$~EA is a randomized $\XP$-algorithm for
the inner point parameterization of De{\u{\i}}neko et
al.~\cite{Deineko2006inner}. It remains an open question whether or
not the algorithm (without significant modification) is also a
randomized fpt-algorithm.

\bibliographystyle{plain}
\bibliography{references,frank}

\end{document}